\documentclass[sigconf]{acmart}
\AtBeginDocument{%
  }

\usepackage{algorithm}
\usepackage{algorithmic}
\usepackage{amsmath}
\usepackage{amsthm}  
\usepackage{xcolor}
\usepackage{subcaption}
\usepackage{booktabs}
\usepackage{multirow}
\usepackage{pifont}
\usepackage{caption}
\usepackage[labelformat=simple]{subcaption}
\usepackage{marginnote}
\usepackage[most]{tcolorbox} 
\usepackage{enumitem}
\usepackage{threeparttable}
\usepackage{makecell}
\usepackage{colortbl}

\newcounter{todoCounter} 

\newcommand{\stitle}[1]{\vspace{2mm}\noindent\textbf{#1}}

\theoremstyle{definition}
\newtheorem{definition}{Definition}[section]
\theoremstyle{plain}
\newtheorem{theorem}{Theorem}[section]

\setcopyright{acmlicensed}
\copyrightyear{2025}
\acmYear{2025}
\acmDOI{XXXXXXX.XXXXXXX}
\acmConference[Conference acronym 'XX]{Make sure to enter the correct
  conference title from your rights confirmation email}{June 03--05,
  2025}{Woodstock, NY}
\acmISBN{978-1-4503-XXXX-X/2025/06}

\begin{document}

\title{CORE: Contrastive Masked Feature Reconstruction on Graphs}

\author{Jianyuan Bo}
\affiliation{%
  \institution{Singapore Management University}
  \city{Singapore}
  \country{Singapore}}
\email{jybo.2020@phdcs.smu.edu.sg}

\author{Yuan Fang}
\affiliation{%
  \institution{Singapore Management University}
  \city{Singapore}
  \country{Singapore}}
\email{yfang@smu.edu.sg}



\begin{abstract}
In the rapidly evolving field of self-supervised learning on graphs, generative and contrastive methodologies have emerged as two dominant approaches. Our study focuses on \emph{masked feature reconstruction} (MFR), a generative technique where a model learns to restore the raw features of masked nodes in a self-supervised manner. We observe that both MFR and graph contrastive learning (GCL) aim to maximize agreement between similar elements. Building on this observation, we reveal a novel theoretical insight: under specific conditions, the objectives of MFR and node-level GCL converge, despite their distinct operational mechanisms. This theoretical connection suggests these approaches are complementary rather than fundamentally different, prompting us to explore their integration to enhance self-supervised learning on graphs. Our research presents Contrastive Masked Feature Reconstruction (CORE), a novel graph self-supervised learning framework that integrates contrastive learning into MFR. Specifically, we form positive pairs exclusively between the original and reconstructed features of masked nodes, encouraging the encoder to prioritize contextual information over the node's own features. Additionally, we leverage the masked nodes themselves as negative samples, combining MFR's reconstructive power with GCL's discriminative ability to better capture intrinsic graph structures. Empirically, our proposed framework CORE significantly outperforms MFR across node and graph classification tasks, demonstrating state-of-the-art results. In particular, CORE surpasses GraphMAE and GraphMAE2 by up to 2.80\% and 3.72\% on node classification tasks, and by up to 3.82\% and 3.76\% on graph classification tasks.
\end{abstract}




\maketitle

\section{Introduction}
Self-supervised learning (SSL) on graphs has become instrumental to graph representation learning. Typically, a graph neural network (GNN)~\cite{velivckovic2018graph} is trained on the inherent patterns within graphs without any labeled data. The learned representations can be subsequently utilized for various downstream tasks, such as node classification \cite{velickovic2019deep,zhu2020deep,thakoor2022large}, link prediction \cite{kipf2016variational,jiang2021contrastive,li2023s}, and graph classification \cite{sun2019infograph,you2020graph,xia2022simgrace}. 

Amidst these developments, SSL on graphs diverges into two broad groups: contrastive and generative approaches. Graph contrastive learning (GCL) methods focus on pulling positive samples closer and distancing negative ones. For instance, GRACE~\cite{zhu2020deep}, a node-level contrastive learning method, generates graph views with various data augmentations. It treats corresponding nodes of the augmented views as positive samples, while considering all other nodes within the same graph as negative samples. Meanwhile, generative approaches learn by predicting the missing or masked parts of the graph. VGAE~\cite{kipf2016variational} is a pioneering and popular example of generative methods, which reconstructs the adjacency matrix using the inner product of the latent node representations.

Though contrastive methods have demonstrated effectiveness in graph learning for various graph tasks, they encounter challenges such as the need for robust graph augmentations~\cite{sugrl2022} and the associated computational overhead~\cite{thakoor2022large}. Recent graph SSL has seen a significant shift towards generative methods, with \textit{masked feature reconstruction} (MFR) emerging as a particularly powerful technique demonstrated in GraphMAE~\cite{hou2022graphmae} and GraphMAE2~\cite{hou2023graphmae2}. MFR augments the graph by masking raw node features and then challenging the model to reconstruct these masked nodes based on the remaining graph structure. Notably, GraphMAE outperform not only other generative methods but also surpass most contrastive learning methods across both node and graph classification tasks without requiring extensive augmentation tuning.

It is intriguing that both MFR and GCL aim to maximize the agreement between similar elements. MFR reconstructs the original features of masked nodes by maximizing agreement between the masked nodes' raw features and their reconstructed embeddings, while contrastive methods form positive pairs between similar graph elements, such as nodes and graphs, and maximize the agreement between them. Given these similarities, this raises the question: (1) \textit{Is MFR compatible with GCL, and can we integrate them to leverage the benefits of both worlds?} Our investigation reveals a crucial insight: \textit{the objectives of MFR and node-level GCL converge} under certain conditions, when the temperature parameter $T$ in the contrastive loss \cite{chen2020simple} approaches zero, amounting to not using negative samples explicitly. Before presenting the theoretical analysis in Sect.~\ref{sec:theoretical_insights}, we conducted preliminary experiments that empirically support this insight (Sect.~\ref{sec:prelim_analysis}). As the temperature decreases, the performance of node-level GCL aligns more closely with that of MFR. The trend suggests that at lower temperatures, positive pairs become more dominant in the loss function, converging to a scenario where negative samples are minimally or not utilized. This phenomenon underlines the diminishing distinction between contrastive and generative methods under certain temperature conditions, effectively positioning MFR as a unique subset of node-level GCL that does not require explicit negative samples. 

In particular, integrating MFR into contrastive learning would allow the exploitation of negative samples without the augmentation overhead associated with contrastive learning. Essentially, we can ground both positive and negative pairs on the masked nodes, without requiring elaborate augmentation strategies.  
While the reconstructed masked nodes and their original features naturally constitute the positive pairs, we wonder a crucial question: (2) \textit{How can we leverage masked nodes to form effective negative samples within MFR?} Masked nodes, while distinct in their raw features, become indistinguishable once masked. By forming negative samples based on these masked nodes, we encourage the model to deeply understand and reconstruct their original features, effectively capturing the intrinsic graph structures. Our empirical findings suggest that employing masked nodes as negative samples significantly enhances the model's generalization and discrimination capabilities across various graph-based tasks.

In response to these questions, we introduce \textbf{Co}ntrastive Masked Feature \textbf{Re}construction (CORE), a novel self-supervised learning (SSL) approach for graphs. CORE innovatively integrates contrastive learning with masked feature reconstruction by leveraging masked nodes as both positive and negative samples. By forming positive pairs exclusively between the original and reconstructed features of masked nodes, and then using other masked nodes as negative samples, CORE ensures that the encoder prioritizes contextual information while effectively distinguishing between similar nodes. Our contributions are threefold: (1) We provide a novel perspective that aligns masked feature reconstruction with node-level graph contrastive learning through both theoretical and empirical analysis. (2) We introduce CORE, a new SSL framework that seamlessly integrates feature reconstruction and contrastive learning using masked nodes. (3) We validate our approach through extensive experiments on node and graph classification tasks, demonstrating significant performance improvements, thus substantiating our theoretical insights.

\section{Related Work}
Self-supervised graph learning is increasingly recognized for enabling models to capture essential patterns in graph data without labeled inputs. This approach is typically categorized into contrastive and generative methods, each offering distinct strategies for enhancing representation learning. (1)~\textbf{Contrastive methods} refine representations by contrasting positive and negative samples across various graph elements. DGI~\cite{velickovic2019deep} maximizes mutual information between nodes and graphs, while GRACE~\cite{zhu2020deep} and GCC~\cite{qiu2020gcc} focus on node agreement across different augmented views. Graph-level methods like GraphCL~\cite{you2020graph} and SimGRACE~\cite{xia2022simgrace} maximize agreement between graph embeddings. GCA~\cite{zhu2021graph} introduces integrity-preserving augmentations, and JOAO~\cite{you2021graph} automates augmentation sampling. Unlike previous methods, our approach integrates MFR into contrastive learning, leveraging masked nodes as both anchors and negatives. This eliminates the need for complex augmentations, simplifying the learning process. (2)~\textbf{Generative methods} focus on structure and feature reconstruction. Structure reconstruction methods like VGAE~\cite{kipf2016variational}, GPT-GNN~\cite{hu2020gpt}, and MaskGAE~\cite{li2023s} progress from reconstructing node relationships to capturing complex structural patterns. Feature reconstruction methods, such as LaGraph~\cite{xie2022self}, GraphMAE~\cite{hou2022graphmae}, and GraphMAE2~\cite{hou2023graphmae2}, predict raw node features, employing node masking and advanced reconstruction metrics. These methods emphasize feature reconstruction’s effectiveness over structure-based approaches. Diverging from these, our model builds on theoretical insights about the convergence of MFR and contrastive learning objectives to create an integrated framework that leverages the strengths of both approaches. 

It is worth noting the existing approaches that combine generative and contrastive learning in graph SSL. SGL~\cite{zhu2023sgl} combines graph-level GCL and GraphMAE through a two-branch architecture, while GCMAE~\cite{GCMAE} develops a similar two-branch framework that treats masked autoencoder (MAE) and contrastive learning as separate, complementary learning strategies. Other recent works have explored incorporating uniformity constraints into masked autoencoders. U-MAE~\cite{zhang2022mask} addresses feature collapse by applying a uniformity loss on encoder outputs to spread out learned embeddings in the latent space. Similarly, AUG-MAE~\cite{wang2024rethinking} extends this concept to graphs, applying uniformity regularization on encoder representations alongside an adversarial masking strategy.

Our approach differs fundamentally from both two-branch frameworks and uniformity-based methods. Rather than treating contrastive and generative learning as separate branches or applying additional constraints in the latent space, CORE directly incorporates contrastive learning into the feature reconstruction process. By explicitly integrating MFR into a full contrastive learning framework in the \emph{feature space}, we leverage both positive and negative pairs formed from masked nodes. This enables our model to simultaneously benefit from the ``pull" effect of aligning reconstructed features with their originals and the ``push" effect of discriminating between different masked nodes, leading to more context-rich representations.

\section{Preliminaries}
\stitle{Node-level Graph Contrastive Learning.}
\label{sec:prelim_contrastive}
In the domain of graph contrastive learning, traditional node-level GCL, as exemplified by GRACE~\cite{zhu2020deep}, focuses on maximizing alignment between node embeddings derived from two different augmented views of input graph $\mathcal{G} = \{\mathbf{X}, \mathbf{A}\}$, where $\mathbf{X} \in \mathbb{R}^{N \times F}$ represents raw node features and $\mathbf{A} \in \{0, 1\}^{N \times N}$ is the adjacency matrix. Two augmented views of $\mathcal{G}$, $\widetilde{\mathcal{G}}_A = \{\mathbf{X}_A, \mathbf{A}_A\}$ and $\widetilde{\mathcal{G}}_B=\{\mathbf{X}_B, \mathbf{A}_B\}$, are generated via augmentations $\tau_1$ and $\tau_2$, affecting both node features and graph structure. Node embeddings for these augmented graphs are often obtained through a shared encoder $f_E(\cdot)$ and a projection head $f_D(\cdot)$. Specifically, for $\widetilde{\mathcal{G}}_A$, the node embedding matrix $\mathbf{H}_A = f_E(\mathbf{X}_A, \mathbf{A}_A)$ and projected node embedding $\mathbf{U} = f_D(\mathbf{H}_A)$. Similarly, for $\widetilde{\mathcal{G}}_B$, $\mathbf{H}_B = f_E(\mathbf{X}_B, \mathbf{A}_B)$ and $\mathbf{V} = f_D(\mathbf{H}_B)$. The contrastive objective encourages the model to distinguish between positive samples, i.e., pairs of embeddings $\boldsymbol{u}_i \in \mathbf{U}$ and $\boldsymbol{v}_i \in \mathbf{V}$ corresponding to the $i$-th node in two views, and negative samples, which encompass all other nodes within and across these views. The pairwise objective function $\ell\left(\boldsymbol{u}_i, \boldsymbol{v}_i\right)$ as detailed in Eq.~\eqref{eq:pairwise_obj}, employing cosine similarity $ \theta(\cdot, \cdot) $ as a measure of alignment, adjusted by temperature parameter $T$~\cite{chen2020simple}. 
\begin{equation}
\small
\begin{split}
\ell(\boldsymbol{u}_i, \boldsymbol{v}_i) 
= \log \Bigg( e^{\theta(\boldsymbol{u}_i, \boldsymbol{v}_i)/ T} /\Bigg(\underbrace{e^{\theta(\boldsymbol{u}_i, \boldsymbol{v}_i) / T}}_{\text{positive pair}} + \\
\underbrace{\sum_{k=1}^N \mathbf{1}_{[k \neq i]} e^{\theta(\boldsymbol{u}_i, \boldsymbol{v}_k)/ T}}_{\text{inter-view negative pairs}} + 
\underbrace{\sum_{k=1}^N \mathbf{1}_{[k \neq i]} e^{\theta(\boldsymbol{u}_i, \boldsymbol{u}_k)/ T}}_{\text{intra-view negative pairs}}\Bigg) \Bigg)
\end{split}
\label{eq:pairwise_obj}
\end{equation}
where $\mathbf{1}_{[k \neq i]}$ is an indicator function. And the overall contrastive loss averages over all positive pairs within the graph:
\begin{equation}
    \small
     \mathcal{L}_{\text {contrast}}=-\frac{1}{2N} \sum_{i=1}^{N}\left[\ell\left(\boldsymbol{u}_i, \boldsymbol{v}_i\right)+\ell\left(\boldsymbol{v}_i, \boldsymbol{u}_i\right)\right]
     \label{eq:contrastive_loss}
\end{equation}
\stitle{Masked Feature Reconstruction.}
\label{sec:prelim_sce}
In this work, we focus on masked feature reconstruction proposed by GraphMAE~\cite{hou2022graphmae} with scaled cosine error (SCE) since mean squared error (MSE) could suffer from the issues of sensitivity and low selectivity \cite{hou2022graphmae}. Consider a subset of nodes in $\mathcal{G}=\{\mathbf{X}, \mathbf{A}\}$, denoted as $\widetilde{\mathcal{V}} \subset V$, whose features are selectively masked by a learnable token $[\text{MASK}]$, resulting in a corrupted feature representation $\widetilde{\mathbf{X}}$. This corrupted graph, represented as $(\widetilde{\mathbf{X}}, \mathbf{A})$, is then fed into the encoder $f_E(\cdot)$ to generate node representations $\mathbf{H} = f_E(\widetilde{\mathbf{X}}, \mathbf{A})$. Subsequently, the decoder $f_D(\cdot)$ processes $\mathbf{H}$ to predict the masked features, resulting in the output $\mathbf{Z}$. The reconstruction objective $ \mathcal{L}_{\text{rec}} $ aims to minimize the cosine similarity between the raw and predicted features for nodes in $\widetilde{\mathcal{V}}$:
\begin{equation}
    \small
    \mathcal{L}_{\text {rec }} = \frac{1}{|\widetilde{\mathcal{V}}|} \sum_{v_i \in \widetilde{\mathcal{V}}} \left(1 - \frac{\boldsymbol{x}_i^\top \boldsymbol{z}_i}{\|\boldsymbol{x}_i\| \cdot \|\boldsymbol{z}_i\|}\right)^\gamma
    \label{eq:rec_loss}
\end{equation}
where $ \boldsymbol{x}_i $ is the raw feature of node $v_i$, $\boldsymbol{z}_i$ is the predicted feature the node obtained through the decoder $f_D$ with $\mathbf{Z} = f_D(\mathbf{H})$, and $ \gamma \geq 1 $ is a scaling factor. 

\section{Proposed Method}
In this section, we present our theoretical insight into the relationship between MFR and node-level GCL, demonstrating how their objectives converge under certain conditions despite their distinct approaches. Building on their mathematical compatibility, we propose \textbf{Co}ntrastive Masked Node \textbf{Re}construction (CORE), a framework that integrates contrastive learning into masked feature reconstruction to leverage the strengths of both approaches.

\subsection{Theoretical Connection between MFR and Node-level GCL}
\label{sec:theoretical_insights}
This section examines the relationship between masked feature reconstruction (MFR) and node-level graph contrastive learning (GCL), revealing their theoretical compatibility despite their distinct operational mechanisms.

\begin{definition}[Convergence Conditions]\label{def:conditions}
Define a set of conditions $C$ as follows. (i) The temperature parameter $T$ approaches zero ($T \rightarrow 0$). (ii) The augmentation set only includes two operations, $\tau = \{\text{Masking}, \text{No-Operation}\}$, with each applied separately to generate two distinct views, where `Masking' randomly masks certain node features, and `No-Operation' leaves the node features unchanged, and a set of anchor nodes are exclusively formed by the set of masked nodes $\widetilde{V}$. (iii) The anchor nodes $\widetilde{V}$ in the `No-Operation' view can be reconstructed with fidelity approaching the theoretical maximum, i.e., $\mathbf{X}'_{\widetilde{V}} \rightarrow \mathbf{X}_{\widetilde{V}}$, where the similarity between the decoder's output $\mathbf{X}'_{\widetilde{V}}$ and their original raw features $\mathbf{X}_{\widetilde{V}}$ approaches the ideal score, indicating a near-identical reconstruction.\qed
\end{definition}

\begin{theorem}[Convergence of MFR and GCL Objectives]\label{th:convergence}
Under conditions $C$ in Def.~\ref{def:conditions}, the objectives of MFR and node-level GCL converge:
$\mathcal{L}_{\text{contrast}} = \mathcal{L}_{\text{rec}}$,
where $\mathcal{L}_{\text{contrast}}$ and $\mathcal{L}_{\text{rec}}$ are defined in Eqs.~\eqref{eq:contrastive_loss} and \eqref{eq:rec_loss}, respectively.
\qed
\end{theorem}

    \begin{proof}[Sketch]
        We establish the convergence between MFR and node-level GCL objectives through a detailed three-step analysis:
        
        \textbf{Step 1: Analysis of contrastive loss under limiting temperature.}
        With the condition $T \rightarrow 0$, we begin by establishing that for any masked node $v_i$, the similarity between its positive pair $\theta\left(\boldsymbol{u}_i, \boldsymbol{v}_i\right)$ exceeds that of any negative pair $\theta\left(\boldsymbol{u}_i, \boldsymbol{w}\right)$~\cite{zhu2020deep}, where $\boldsymbol{w} \in Neg_i$ represents a node within the set of negative samples for $v_i$:
        \begin{equation}
            \theta\left(\boldsymbol{u}_i, \boldsymbol{w}\right) < \theta\left(\boldsymbol{u}_i, \boldsymbol{v}_i\right)
        \end{equation}
        
        As $T$ approaches zero, this similarity difference is exponentially amplified in the contrastive objective. Specifically, we can analyze the ratio of the exponential terms:
        \begin{equation}
        \lim_{T \rightarrow 0}\frac{e^{{\theta\left(\boldsymbol{u}_i, \boldsymbol{w}\right)}/{T}}}{e^{{\theta\left(\boldsymbol{u}_i, \boldsymbol{v}_i\right)}/{T}}} = \lim_{T \rightarrow 0} e^{\left({\theta\left(\boldsymbol{u}_i, \boldsymbol{w}\right)} - \theta\left(\boldsymbol{u}_i, \boldsymbol{v}_i\right)\right)/{T}} = 0
        \end{equation}
        
        Since the exponent ${\theta\left(\boldsymbol{u}_i, \boldsymbol{w}\right)} - \theta\left(\boldsymbol{u}_i, \boldsymbol{v}_i\right)$ is negative and divided by a value approaching zero, the entire expression approaches zero. This means that the contribution from negative pairs becomes negligible compared to the positive pair when $T$ is very small.
        
        \textbf{Step 2: Reformulation and simplification of the pairwise objective.}
        We now analyze how this limiting behavior affects the pairwise objective function from Eq.~\eqref{eq:pairwise_obj}. By taking the logarithm of the softmax expression, we get:
        \begin{equation}
        \begin{aligned}
        \max \ell\left(\boldsymbol{u}_i, \boldsymbol{v}_i\right) &= \max \lim_{T \rightarrow 0} \log \left(\frac{e^{{\theta\left(\boldsymbol{u}_i, \boldsymbol{v}_i\right)}/{T}}}{e^{{\theta\left(\boldsymbol{u}_i, \boldsymbol{v}_i\right)}/{T}}+\sum_{\boldsymbol{w} \in Neg_i} e^{{\theta\left(\boldsymbol{u}_i, \boldsymbol{w}\right)}/{T}}}\right) \\
        &= \max \lim_{T \rightarrow 0} \log \left(\frac{1}{1+\sum_{\boldsymbol{w} \in Neg_i} \frac{e^{{\theta\left(\boldsymbol{u}_i, \boldsymbol{w}\right)}/{T}}}{e^{{\theta\left(\boldsymbol{u}_i, \boldsymbol{v}_i\right)}/{T}}}}\right) \\
        &= \min \lim_{T \rightarrow 0} \log \left(1 + \sum_{\boldsymbol{w} \in Neg_i} \frac{e^{{\theta\left(\boldsymbol{u}_i, \boldsymbol{w}\right)}/{T}}}{e^{{\theta\left(\boldsymbol{u}_i, \boldsymbol{v}_i\right)}/{T}}}\right)
        \end{aligned}
        \end{equation}
        
        Since each term in the summation approaches zero as shown in Step 1, the summation also approaches zero. Therefore, maximizing the pairwise objective becomes equivalent to minimizing $\log(1 + 0) = 0$, which is a constant. This means that the optimization effectively becomes unconstrained by the negative samples, and the objective simplifies to:
        \begin{equation}
        \max \ell\left(\boldsymbol{u}_i, \boldsymbol{v}_i\right) \equiv \max \theta(\boldsymbol{u}_i, \boldsymbol{v}_i)
        \end{equation}
        
        Thus, under the condition $T \rightarrow 0$, maximizing the contrastive objective becomes equivalent to directly maximizing the cosine similarity between the positive pairs, effectively ignoring the negative samples.
        
        \textbf{Step 3: Connecting GCL and MFR under the specified augmentation set.}
        Given the augmentation set $\tau$, we can now establish a concrete mapping between the GCL and MFR frameworks:
        
        \begin{itemize}
        \item For $\widetilde{\mathcal{G}}_A$ (the graph view with `No-Operation' augmentation): The node embeddings $\boldsymbol{u}_i$ correspond to the $i$-th node embedding from $\mathbf{X}' = f_D(f_E(\mathbf{X}, \mathbf{A}))$, which are the reconstructed features of the original graph.
        
        \item For $\widetilde{\mathcal{G}}_B$ (the graph view with `Masking' augmentation): The node embeddings $\boldsymbol{v}_i$ correspond to $\boldsymbol{z}_i$ from $\mathbf{Z} = f_D(f_E(\widetilde{\mathbf{X}}, \mathbf{A}))$, which are the reconstructed features of the masked graph.
        \end{itemize}
        

        Under this mapping, the contrastive loss for masked nodes becomes:
        \begin{equation}
        \begin{aligned}
        \mathcal{L}_{\text{contrast}} &= -\frac{1}{|\widetilde{\mathcal{V}}|} \sum_{v_i\in\widetilde{\mathcal{V}}} \ell\left(\boldsymbol{u}_i, \boldsymbol{v}_i\right) \\
        &= \frac{1}{|\widetilde{\mathcal{V}}|} \sum_{v_i\in\widetilde{\mathcal{V}}} \left(1 - \frac{\boldsymbol{x}'_i \cdot \boldsymbol{z}_i}{\| \boldsymbol{x}'_i \| \cdot \| \boldsymbol{z}_i \|}\right)
        \end{aligned}
        \end{equation}
        
        When comparing with the MFR loss in Eq.~\eqref{eq:rec_loss} (with $\gamma=1$ for simplicity), we can express both objectives in a parallel form:
        \begin{equation}
            \mathcal{L}_{\text{rec}} = 1 - \theta(\mathbf{X}_{\widetilde{V}}, \mathbf{Z}_{\widetilde{V}})
            \label{eq:rec_simple}
        \end{equation}
        \begin{equation}
            \mathcal{L}_{\text{contrast}} = 1 - \theta(\mathbf{X}'_{\widetilde{V}}, \mathbf{Z}_{\widetilde{V}})
            \label{eq:contrast_simple}
        \end{equation}
        
        Here, $\mathbf{X}_{\widetilde{V}}$ represents the original features of the masked nodes, $\mathbf{X}'_{\widetilde{V}}$ represents the reconstructed features of the masked nodes from the unmasked graph, and $\mathbf{Z}_{\widetilde{V}}$ represents the reconstructed features of the masked nodes from the masked graph.
        
        The key observation is that when condition (iii) in Definition~\ref{def:conditions} is satisfied, i.e., when $\mathbf{X}'_{\widetilde{V}} \rightarrow \mathbf{X}_{\widetilde{V}}$ (when the encoder-decoder can perfectly reconstruct the features of unmasked nodes), we have:
        \begin{equation}
        \lim_{\mathbf{X}'_{\widetilde{V}} \rightarrow \mathbf{X}_{\widetilde{V}}}{\mathcal{L}_{\text{contrast}}} = \mathcal{L}_{\text{rec}}
        \end{equation}
        This equivalence underscores a convergence between contrastive and reconstruction loss functions under ideal conditions of feature reconstruction.
        \end{proof}

The theoretical result reveals two key insights: (1) \textbf{Implicit Nature of MFR}: The theorem demonstrates that MFR implicitly operates as a node-level contrastive learning method without explicit negative samples (equivalent to setting near-zero temperature). While this enables effective feature reconstruction, it may limit the ability of the model to learn discriminative representations that distinguish between different nodes. (2). \textbf{Compatible Objectives}: Despite their difference, where GCL focuses on relative sample positioning in the latent space and MFR focuses on accurate feature reconstruction, their objectives can align under certain conditions. This alignment suggests that these methods are not fundamentally incompatible but rather synergistic, motivating their integration to leverage both discriminative and reconstructive capabilities.

These insights inspire our framework design: rather than operating at the theoretical limit where $T \rightarrow 0$, we maintain a standard temperature while introducing negative samples into the MFR process. This approach preserves the strengths of both methods, MFR's ability to capture node features and GCL's discriminative power, leading to more robust graph representations.

It is important to note that while our theoretical analysis relies on the conditions in Definition~\ref{def:conditions}, particularly $T \rightarrow 0$ and the perfect reconstruction requirement, these represent idealized theoretical conditions that facilitate mathematical analysis rather than practical implementation constraints. The value of this theoretical result lies in revealing the fundamental connection between two seemingly distinct approaches, providing conceptual justification for their integration. In practice, as we will show in our experiments (Fig.~\ref{fig:temperatures}), operating at very low temperatures actually degrades performance, reinforcing that our framework's strength comes from deliberately moving away from the theoretical limit to leverage the complementary benefits of both approaches.

\subsection{Contrastive Masked Feature Reconstruction}
\label{sec:core}

\begin{figure*}[t]
\centering
\includegraphics[width=0.9\linewidth]{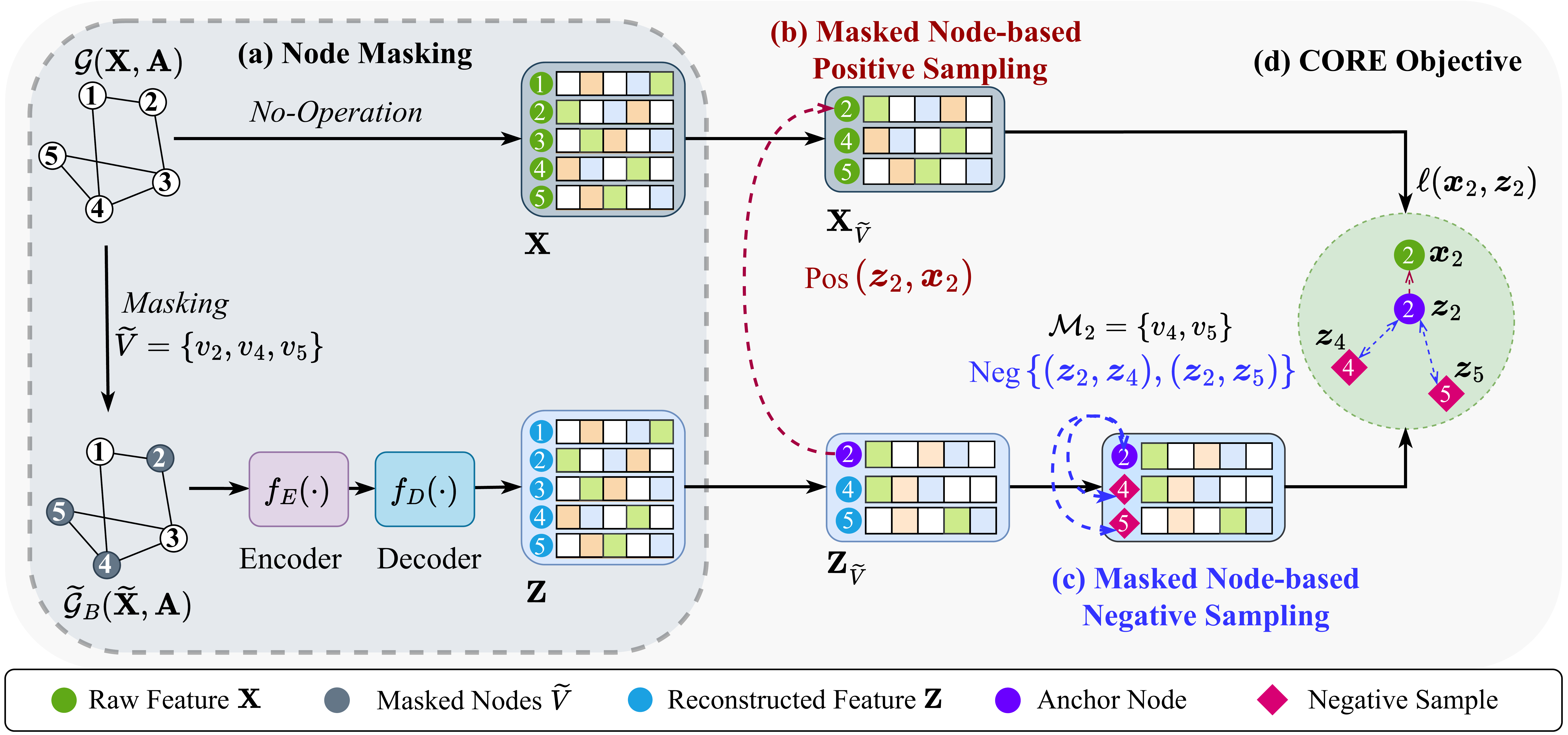}
\caption{Overview of the CORE framework. 
(For clarity, we only show one anchor node here; the pairwise objective
 is applied across all masked nodes in the graph.)}
\label{fig:core}
\end{figure*}

Building upon the theoretical insights outlined in Sect.\ref{sec:theoretical_insights}, we introduce our \textbf{Co}ntrastive Masked Feature \textbf{Re}construction (CORE), a novel graph SSL framework motivated by the idea of integrating the generative mechanism of MFR with the contrastive approach of GCL. By leveraging masked nodes as both positive and negative samples within the same framework, CORE reduces the reliance on complex augmentations, thereby enhancing the model's ability to capture intrinsic graph structures while boosting the performance of masked feature reconstruction. A key feature of our framework lies in its ability to seamlessly integrate with any existing MFR framework, including the state-of-the-art GraphMAE and its successors, while preserving their specific techniques that encompass both the node masking process and the encoder-decoder architecture, as shown in Fig.~\ref{fig:core}(a).

\stitle{Design choices.} To materialize our CORE framework, we detail two pivotal design choices that underpin our approach. These elements collectively contribute to the robustness of CORE, facilitating a synergistic integration of contrastive learning while preserving the fundamental aspects of MFR.

\emph{\textbf{1) Masked node-based positive sampling.}} 
Conventional contrastive methods often use all nodes as anchor points~\cite{zhu2020deep,thakoor2022large}. This is equivalent to form positive pairs between the predictions of all nodes (outputs from the encoder-decoder architecture) and their raw features in the MFR framework. However, forming positive pairs on unmasked nodes might discourage the encoder from leveraging contextual information to reconstruct the nodes, thereby diminishing the overall efficacy in capturing relevant information from neighboring nodes~\cite{xie2022self}. Therefore, we form the positive pairs exclusively on the masked nodes to encourage the encoder to \textit{prioritize contextual information} over the node's own features. In our framework, each positive pair consists of the reconstruction for a masked node alongside its raw, unmasked features $(\boldsymbol{z}_i, \boldsymbol{x}_i)$ as depicted in Fig.~\ref{fig:core}(b). It refines the contrastive learning process, ensuring a focused engagement with the original objective of masked feature reconstruction. 

\emph{\textbf{2) Masked node-based negative sampling.}} Diverging from other contrastive methods \cite{zhu2020deep,zhu2021graph} that treat all other nodes as negative samples, our second design choice, depicted in Fig.~\ref{fig:core}(c), exclusively selecting negative samples from the set of masked nodes, $\widetilde{\mathcal{V}}$, to form negative pair $(\boldsymbol{z}_i, \boldsymbol{z}_k)$ for $v_i$. This strategy has two following advantages. (i) \textit{Enhanced feature reconstruction}: Masked nodes, which are distinct in their raw features, become indistinguishable once masked, especially when masked with the same learnable token as typically employed \cite{hou2022graphmae,hou2023graphmae2}. This characteristic makes them ideal negative samples, as differentiating these nodes requires the model to deeply understand and reconstruct their raw features. (ii) \textit{Context-rich representation learning}: Encouraging differentiation among the masked nodes pushes the model to learn robust representations that accurately reflect a node's neighborhood and its contextual information, making the reconstructed features not just precise but contextually rich, reflecting the intricate relationships within the graph structure.

\stitle{Outline of CORE.}
Our framework begins with node masking, depicted in Fig.~\ref{fig:core}(a), where a subset of nodes $\widetilde{V}$ in the input graph $\mathcal{G}=\{\mathbf{X}, \mathbf{A}\}$ is selectively masked, creating a masked graph view $\widetilde{\mathcal{G}}_B(\widetilde{\mathbf{X}}, \mathbf{A})$. 
$\widetilde{\mathcal{G}}_B$ is subsequently processed by an encoder-decoder architecture, where the encoder produces latent representations of the nodes and the decoder reconstructs the masked nodes to $\mathbf{Z}_{\widetilde{V}}$. 
In alignment with our \textit{masked node-based positive sampling} strategy as shown in Fig.~\ref{fig:core}(b), we form the positive pairs composed of the reconstructed features $\mathbf{Z}_{\widetilde{V}}$ of the masked nodes and their corresponding raw features $\mathbf{X}_{\widetilde{V}}$ for each node in the set $\widetilde{V}$. 
Then, as shown in Fig.~\ref{fig:core}(c), our \textit{masked node-based negative sampling} strategy randomly selects $|\mathcal{M}_i|$ negative samples solely from the masked nodes for each anchor node, where $\mathcal{M}_i \subseteq \widetilde{\mathcal{V}} \setminus \{v_i\}$ and $|\mathcal{M}_i| \ll |\widetilde{\mathcal{V}}|$, providing diverse and varied negative examples for the model. Note that each subset $\mathcal{M}_i$ has an identical, predetermined size across all nodes. Finally, we amplify the similarity between each positive pair while distinguishing the reconstructed features from the selected negative samples.

\stitle{Loss formulation.} Having outlined our designs revolving around \textit{masked node-based positive sampling} and \textit{masked node-based negative sampling}, we now formalize the loss. Following the notations defined in Sect.~\ref{sec:prelim_sce}, where $\boldsymbol{z}_i$ denotes the reconstructed feature vector for masked node $v_i$ and $\boldsymbol{x}_i$ represents its raw features, we present the proposed pairwise objective below:

\begin{equation}
\small
\begin{split}
\ell\left(\boldsymbol{z}_i, \boldsymbol{x}_i\right)=\log \Bigg(e^{\theta\left(\boldsymbol{z}_i, \boldsymbol{x}_i\right)/ T} / \Bigg(\underbrace{e^{\theta\left(\boldsymbol{z}_i, \boldsymbol{x}_i\right) / T}}_{\text {positive pair}}+ \\
\underbrace{\sum_{k=1}^{|\mathcal{M}_i|} \mathbf{1}_{[k \neq i]} e^{\theta\left(\boldsymbol{z}_i, \boldsymbol{z}_{k}\right) / T}}_{\text{negative pairs}}\Bigg)\Bigg), \forall v_i \in \widetilde{V} \text{ and } \mathcal{M}_i \subseteq \widetilde{\mathcal{V}}.
\end{split}
\label{eq:core}
\end{equation}

The essence of this loss lies in its dual objectives. Firstly, to maximize the agreement between the positive pair, i.e., the reconstructed and raw features of the masked node ($\boldsymbol{z}_i$ and $\boldsymbol{x}_i$) as in the original masked feature reconstruction. Secondly, to ensure each node's representation is distinct from those of other nodes ($\boldsymbol{z}_i$ and $\boldsymbol{z}_k$), we leverage negative samples drawn exclusively from the masked nodes. This focused approach ensures that the learning process is highly specific to the task of reconstructing masked nodes, leading to more robust and discriminative node representations.

The pairwise objective will be further aggregated across all masked nodes, leading to the overall objective as follows.
\begin{equation}
    \mathcal{L} = -\frac{1}{|\widetilde{V}|} \sum_{i=1}^{|\widetilde{V}|}\ell\left(\boldsymbol{z}_i, \boldsymbol{x}_i\right)
\end{equation}

By focusing on masked nodes and utilizing them as both positive and negative samples, CORE enhances masked feature reconstruction in a way that not only preserves the effectiveness of feature restoration but also improves the model's ability to capture intrinsic graph structures. This integration of generative methods with contrastive techniques leverages the strengths of both approaches, effectively addressing challenges such as the potential neglect of contextual information in feature reconstruction and the complexities of data augmentation. As a result, CORE offers a more robust and efficient approach to self-supervised graph learning.

\section{Experiments}
This section presents the empirical study evaluating the effectiveness of CORE framework. We assess CORE in unsupervised node and graph classification, comparing it against a comprehensive set of baselines. For each task, we follow the same experimental procedures, including data splits and evaluation protocols, as outlined in standard settings \cite{velickovic2019deep,zhang2021canonical,sun2019infograph}.

\stitle{Datasets.} 
For node classification task, we utlize eight standard benchmarks in transductive learning scenarios: Citeseer, Cora, PubMed~\cite{yang2016revisiting}, Wiki-CS~\cite{mernyei2020wiki}, Amazon-Computers (Computer), Ama\-zon-Photos (Photo)~\cite{mcauley2015image}, Coauthor-CS (CS), and Coauthor-Physics (Physics)~\cite{sinha2015overview} enabling a comprehensive evaluation of CORE's performance on various scale of graph. Additionally, the PPI dataset is used in an inductive learning setting, further demonstrating the CORE's versatility across different learning paradigms. For graph classification, we utilize seven benchmark datasets from TUDataset~\cite{morris2020tudataset}, including MUTAG, PROTEINS, NCI1 for scenarios where node labels serve as input features, and IMDB-Binary (IMDB-B), IMDB-Multi (IMDB-M), REDDIT-Binary (REDDIT-B), COLLAB for instances utilizing node degrees. This diverse selection enables a comprehensive analysis of CORE's applicability across different graph structures and classification challenges. 

\stitle{Baselines.}
Our benchmarks span a divers range of self-supervised graph pre-training methods, including both contrastive and generative approaches, with a special focus on GraphMAE~\cite{hou2022graphmae} and GraphMAE2~\cite{hou2023graphmae2}. They serve as direct competitors, since CORE utilized their Node Masking process as shown in Fig.~\ref{fig:core}(a). Importantly, our pairwise objective has been adaptively applied to the latent feature space in GraphMAE2 showcases its versatility.

For node classification, we first benchmark against supervised methods like GCN~\cite{thomas2017gcn} and GAT~\cite{velivckovic2018graph} followed by comparisons with contrastive models such as DGI~\cite{yanardag2015deep}, GMI~\cite{peng2020graph}, MVGRL~\cite{hassani2020contrastive}, GRACE~\cite{zhu2020deep}, BGRL~\cite{thakoor2022large}, CCA-SSG~\cite{zhang2021canonical}, and ABGML~\cite{chen2023graph}. We also include generative models, including GAE~\cite{kipf2016variational}, AUG-MAE~\cite{wang2024rethinking}, with a focus on GraphMAE and GraphMAE2 as our direct competitors. For graph classification, we evaluate against graph kernel methods like Weisfeiler-Lehman sub-tree kernel (WL)~\cite{shervashidze2011weisfeiler} and Deep Graph Kernel (DGK)~\cite{yanardag2015deep}, supervised approaches including GIN~\cite{xu2019powerful} and DiffPool~\cite{ying2018hierarchical}, various contrastive methods, such as GCC~\cite{qiu2020gcc}, graph2vec~\cite{narayanan2017graph2vec}, Infograph~\cite{sun2019infograph}, GraphCL~\cite{you2020graph}, JOAO~\cite{you2021graph}, MVGRL~\cite{hassani2020contrastive}, InfoGCL~\cite{Xu2021InfoGCLIG}, and GCMAE~\cite{GCMAE}. GraphMAE and GraphMAE2 are highlighted as critical benchmarks, reinforcing the direct comparison against our framework. 

\stitle{Settings.} For node classification, we train a GNN encoder using our framework. Following established methods~\cite{velickovic2019deep,hassani2020contrastive,zhang2021canonical,thakoor2022large}, the encoder's parameters are frozen during testing to produce node embeddings, which are then used to train a linear classifier. We evaluate using mean accuracy across 20 initializations on Cora, Citeseer, PubMed, and PPI using public data splits, with random splits (10\%/10\%/80\%) for other datasets. The encoder and decoder use the standard GAT~\cite{velivckovic2018graph} architecture as in GraphMAE. For graph classification, the pre-trained GNN encoder generates graph embeddings via a readout function, which are then used for prediction with a LIBSVM~\cite{chang2011libsvm} classifier. Performance is measured by mean 10-fold cross-validation accuracy over five runs. GIN~\cite{xu2019powerful} is employed as the backbone for both encoder and decoder, consistent with its widespread use in graph classification research. 

\subsection{Node Classification}
Adhering to the tradition of leveraging existing benchmarks for consistency, we report results from previous works with the same experimental setup if available, as detailed in Tab.~\ref{tab:node_clf}. To ensure a fair comparison with CORE, GraphMAE and GraphMAE2 results are re-evaluated on the same machine with provided configurations when available; otherwise, both methods are tuned based on validation performance. Tab.~\ref{tab:node_clf} showcases that CORE improves GraphMAE's performance on eight out of nine datasets in unsupervised node classification. Paired t-tests confirm that these improvements for both GraphMAE ($p = 0.0146$) and GraphMAE2 ($p = 0.0334$) are statistically significant.

Notably, CORE achieves the highest accuracy among the referenced SSL methods on most datasets, particularly excelling in Cora, Citeseer, Wiki-CS, Computer, Photo, and PPI. Although we include various self-supervised methods for reference, our analysis primarily emphasizes GraphMAE, our direct competitor. Besides, CORE with GraphMAE2 outperforms its predecessor consistently showcasing its versatility in adapting both raw and latent feature spaces.

\begin{table*}[t]
    \centering
    \small
    \caption{Experiment results in unsupervised representation learning for node classification. Bolded results indicate superior performance between our proposed method and corresponding GraphMAE framework. We report accuracy (\%) for all datasets.
    }
    \begin{threeparttable}
    \renewcommand\arraystretch{1}
    \resizebox{0.95\linewidth}{!}{
    \begin{tabular}{c|ccccccccc|c}
        \toprule[1.2pt]
         Dataset & Cora & Citeseer & PubMed & Wiki-CS & Computer & Photo & CS & Physics & PPI & Average Rank\\
        \midrule
        GCN &
        81.5{\scriptsize$\pm$0.2} &
        70.3{\scriptsize$\pm$0.4} &
        79.0{\scriptsize$\pm$0.5} &
        77.2{\scriptsize$\pm$0.1} &
        86.5{\scriptsize$\pm$0.5} & 
        92.4{\scriptsize$\pm$0.2} & 
        93.0{\scriptsize$\pm$0.3} & 
        95.7{\scriptsize$\pm$0.2} &
        75.7{\scriptsize$\pm$0.1} &
        - \\
        GAT &  
        83.0{\scriptsize$\pm$0.7} &
        72.5{\scriptsize$\pm$0.7} & 
        79.0{\scriptsize$\pm$0.3} & 
        77.7{\scriptsize$\pm$0.1} & 
        86.9{\scriptsize$\pm$0.3} & 
        92.6{\scriptsize$\pm$0.4} & 
        92.3{\scriptsize$\pm$0.2} & 
        95.5{\scriptsize$\pm$0.2} &
        97.3{\scriptsize$\pm$0.2} &
        - \\
        \midrule
        GAE & 
        71.5{\scriptsize$\pm$0.4} & 
        65.8{\scriptsize$\pm$0.4} & 
        72.1{\scriptsize$\pm$0.5} & 
        70.2{\scriptsize$\pm$0.0} & 
        85.3{\scriptsize$\pm$0.2} & 
        91.6{\scriptsize$\pm$0.1} & 
        90.0{\scriptsize$\pm$0.7} & 
        94.9{\scriptsize$\pm$0.1} & 
        - &
        12.1 \\
        GMI & 
        83.0{\scriptsize$\pm$0.2} & 
        72.4{\scriptsize$\pm$0.2} &
        79.9{\scriptsize$\pm$0.4} &
        74.8{\scriptsize$\pm$0.1} &
        82.2{\scriptsize$\pm$0.3} &
        90.7{\scriptsize$\pm$0.2} &
        - &
        - &
        - &
        10.7\\
        GCA & 
        81.8{\scriptsize$\pm$0.2} & 
        71.9{\scriptsize$\pm$0.4} & 
        81.0{\scriptsize$\pm$0.3} & 
        78.3{\scriptsize$\pm$0.0} & 
        87.9{\scriptsize$\pm$0.3} & 
        92.5{\scriptsize$\pm$0.1} & 
        93.1{\scriptsize$\pm$0.0} & 
        95.7{\scriptsize$\pm$0.0} &
        - &
        7.4\\
        DGI & 
        82.3{\scriptsize$\pm$0.6} &
        71.8{\scriptsize$\pm$0.7} &
        76.8{\scriptsize$\pm$0.6} &
        75.4{\scriptsize$\pm$0.1} &
        84.0{\scriptsize$\pm$0.5} &
        91.6{\scriptsize$\pm$0.2} &
        92.2{\scriptsize$\pm$0.6} &
        94.5{\scriptsize$\pm$0.5} &
        63.8{\scriptsize$\pm$0.2} &
        10.4\\
        MVGRL & 
        83.5{\scriptsize$\pm$0.4} &
        73.3{\scriptsize$\pm$0.5} &
        80.1{\scriptsize$\pm$0.7} &
        77.5{\scriptsize$\pm$0.1} &
        87.5{\scriptsize$\pm$0.1} &
        91.7{\scriptsize$\pm$0.1} &
        92.1{\scriptsize$\pm$0.1} &
        95.3{\scriptsize$\pm$0.0} &
        - &
        8.4\\
        GRACE & 
        81.9{\scriptsize$\pm$0.4} & 
        71.2{\scriptsize$\pm$0.5} &
        80.6{\scriptsize$\pm$0.4} &
        80.1{\scriptsize$\pm$0.5} &
        89.5{\scriptsize$\pm$0.4} &
        92.8{\scriptsize$\pm$0.5} &
        91.1{\scriptsize$\pm$0.2} &
        69.7{\scriptsize$\pm$0.2} &
        - &
        8.6\\  
        BGRL &
        82.7{\scriptsize$\pm$0.6} &
        71.1{\scriptsize$\pm$0.8} &
        79.6{\scriptsize$\pm$0.5} &
        80.0{\scriptsize$\pm$0.1} &
        90.3{\scriptsize$\pm$0.2} &
        93.2{\scriptsize$\pm$0.3} & 
        93.3{\scriptsize$\pm$0.1} & 
        95.7{\scriptsize$\pm$0.0} &
        73.6{\scriptsize$\pm$0.2} &
        5.6\\
        CCA-SSG & 
        84.0{\scriptsize$\pm$0.4} &
        73.1{\scriptsize$\pm$0.3} &
        81.0{\scriptsize$\pm$0.4} &
        75.7{\scriptsize$\pm$0.7} &
        88.7{\scriptsize$\pm$0.3} &
        93.1{\scriptsize$\pm$0.1} &
        93.3{\scriptsize$\pm$0.2} &
        95.4{\scriptsize$\pm$0.1} &
        - &
        6.3\\
        ABGML & 
        - &
        - &
        - &
        78.7{\scriptsize$\pm$0.6} &
        90.2{\scriptsize$\pm$0.3} &
        93.5{\scriptsize$\pm$0.4} &
        93.6{\scriptsize$\pm$0.2} &
        96.1{\scriptsize$\pm$0.1} &
        - &
        3.2\\
        AUG-MAE & 
        84.3{\scriptsize$\pm$0.4} &
        73.2{\scriptsize$\pm$0.4} &
        81.4{\scriptsize$\pm$0.4} &
        - &
        - &
        - &
        - &
        - &
        74.3 {\scriptsize$\pm$0.1}&
        2.3\\
        \midrule \midrule
        GraphMAE & 
        83.7{\scriptsize$\pm$0.8} &
        72.9{\scriptsize$\pm$0.7} & 
        80.6{\scriptsize$\pm$0.7} & 
        79.0{\scriptsize$\pm$0.4} & 
        89.6{\scriptsize$\pm$0.3} & 
        93.0{\scriptsize$\pm$0.3} & 
        92.9{\scriptsize$\pm$0.2} & 
        95.5{\scriptsize$\pm$0.1} &
        71.4{\scriptsize$\pm$0.3}&
        6.6\\
        \rowcolor{lightgray}
        $+$ CORE$^\text{1}$ & 
        \bf 84.1{\scriptsize$\pm$0.5} &
        \bf 74.0{\scriptsize$\pm$0.6} &
        \bf 80.6{\scriptsize$\pm$0.5} &
        \bf 79.6{\scriptsize$\pm$0.5} &
        \bf 90.7{\scriptsize$\pm$0.2} &
        \bf 93.6{\scriptsize$\pm$0.4} &
        \bf 93.0{\scriptsize$\pm$0.1} &
        \bf 95.6{\scriptsize$\pm$0.1} &
        \bf 73.4{\scriptsize$\pm$0.2} &
        \bf 4\\
        \midrule
        GraphMAE2 & 
        84.1{\scriptsize$\pm$0.6} &
        73.2{\scriptsize$\pm$0.5} &
        \bf 81.4{\scriptsize$\pm$0.7} &
        79.1{\scriptsize$\pm$0.5} &
        89.9{\scriptsize$\pm$0.3} &
        93.2{\scriptsize$\pm$0.3} &
        93.1{\scriptsize$\pm$0.2} &
        95.5{\scriptsize$\pm$0.2} &
        72.5{\scriptsize$\pm$0.3} &
        4.6\\
        \rowcolor{lightgray}
        $+$ CORE$^\text{2}$ &
        \bf 84.2{\scriptsize$\pm$0.5} &
        \bf 74.1{\scriptsize$\pm$0.6} &
        81.3{\scriptsize$\pm$0.7} &
        \bf 80.2{\scriptsize$\pm$0.6} &
        \bf 91.0{\scriptsize$\pm$0.3} &
        \bf 93.9{\scriptsize$\pm$0.4} &
        \bf 93.2{\scriptsize$\pm$0.2} &
        \bf 95.6{\scriptsize$\pm$0.1} &
        \bf 75.2{\scriptsize$\pm$0.2} &
        \bf 2.1\\
        \bottomrule[1.2pt]
    \end{tabular}}\\
    \end{threeparttable}
    \begin{tablenotes}
      \footnotesize
      \item For a fair comparison, we \textbf{reran the experiments for GraphMAE and GraphMAE2} using their official implementations.
    \item The reported results of other baselines are from previous papers, rounded to 1 decimal place, if available.
      \item $\textsuperscript{1}$ Our framework utilizes the node masking process proposed by GraphMAE. 
      \item $\textsuperscript{2}$ Our framework utilizes the node masking process proposed by GraphMAE2.
  \end{tablenotes}
    \label{tab:node_clf}
\end{table*}

\begin{table*}[t]
    \centering
    \small
    \caption{Experiment results in unsupervised representation learning for graph classification. Bolded results indicate superior performance between CORE and the corresponding GraphMAE framework. We report accuracy (\%) for all datasets. AR denotes the average rank of SSL methods.}
    \begin{threeparttable}
    \renewcommand\arraystretch{1}
    \resizebox{0.95\linewidth}{!}{
    \begin{tabular}{c|ccccccc|c}
        \toprule[1.2pt]
        Dataset & IMDB-B & IMDB-M & PROTEINS & COLLAB & MUTAG & REDDIT-B & NCI1 & Average Rank\\ 
        \midrule
        GIN &
        75.1{\scriptsize$\pm$5.1} &
        52.3{\scriptsize$\pm$2.8} &
        76.2{\scriptsize$\pm$2.8} &
        80.2{\scriptsize$\pm$1.9} & 
        89.4{\scriptsize$\pm$5.6} &
        92.4{\scriptsize$\pm$2.5} &
        82.7{\scriptsize$\pm$1.7} & 
        - \\
        DiffPool &
        72.6{\scriptsize$\pm$3.9} &  
        -  & 
        75.1{\scriptsize$\pm$3.5} &
        78.9{\scriptsize$\pm$2.3} & 
        85.0{\scriptsize$\pm$10.3} & 
        92.1{\scriptsize$\pm$2.6} & 
        - &
        - \\
        \midrule
        WL & 
        72.3{\scriptsize$\pm$3.4} &
        47.0{\scriptsize$\pm$0.5} &
        72.9{\scriptsize$\pm$0.6} &
        - & 
        80.7{\scriptsize$\pm$3.0} & 
        68.8{\scriptsize$\pm$0.4} &
        80.3{\scriptsize$\pm$0.5} &
        - \\
        DGK & 
        67.0{\scriptsize$\pm$0.6} &
        44.6{\scriptsize$\pm$0.5} &
        73.3{\scriptsize$\pm$0.8} &
        - &
        87.4{\scriptsize$\pm$2.7} &
        78.0{\scriptsize$\pm$0.4} &
        80.3{\scriptsize$\pm$0.5} &
        - \\ 
        \midrule
        graph2vec & 
        71.1{\scriptsize$\pm$0.5} &
        50.4{\scriptsize$\pm$0.9} &
        73.3{\scriptsize$\pm$2.1} &
        - & 
        83.2{\scriptsize$\pm$9.2} &
        75.8{\scriptsize$\pm$1.0} &
        73.2{\scriptsize$\pm$1.8} &
        10.6\\ 
        Infograph & 
        73.0{\scriptsize$\pm$0.9} &
        49.7{\scriptsize$\pm$0.5} &
        74.4{\scriptsize$\pm$0.3} &
        70.7{\scriptsize$\pm$1.1} &
        89.0{\scriptsize$\pm$1.1} &
        82.5{\scriptsize$\pm$1.4} &
        76.2{\scriptsize$\pm$1.1} &
        9.1\\ 
        GraphCL & 
        71.1{\scriptsize$\pm$0.4} &
        48.6{\scriptsize$\pm$0.7} &
        74.4{\scriptsize$\pm$0.4} &
        71.4{\scriptsize$\pm$1.2} &
        86.8{\scriptsize$\pm$1.3} &
        89.5{\scriptsize$\pm$0.8} &
        77.9{\scriptsize$\pm$0.4} &
        9.2\\
        JOAO & 
        70.2{\scriptsize$\pm$3.1} &
        49.2{\scriptsize$\pm$0.8} &
        74.6{\scriptsize$\pm$0.4} & 
        69.5{\scriptsize$\pm$0.4} & 
        87.4{\scriptsize$\pm$1.0} & 
        85.3{\scriptsize$\pm$1.4} &
        78.1{\scriptsize$\pm$0.5} &
        9.7\\
        GCC & 
        72.0 &
        49.4 &
        - & 
        78.9 &
        - & 
        89.8 &
        - &
        8.0\\ 
        MVGRL &
        74.2{\scriptsize$\pm$0.7} & 
        51.2{\scriptsize$\pm$0.5} &
        - & 
        - & 
        89.7{\scriptsize$\pm$1.1} & 
        84.5{\scriptsize$\pm$0.6} & 
        - &
        7.5\\
        InfoGCL &
        75.1{\scriptsize$\pm$0.9} & 
        51.4{\scriptsize$\pm$0.8} &  
        - & 
        80.0{\scriptsize$\pm$1.3} &
        91.2{\scriptsize$\pm$1.3} &
        - & 
        80.2{\scriptsize$\pm$0.6} &
        5.8\\
        GCMAE & 
        75.8{\scriptsize$\pm$0.2} &
        52.5{\scriptsize$\pm$0.5} &
        - & 
        81.3{\scriptsize$\pm$0.3} & 
        91.3{\scriptsize$\pm$0.6} & 
        91.8{\scriptsize$\pm$0.2} & 
        81.4{\scriptsize$\pm$0.3} &
        3.2\\
        AUG-MAE & 
        75.6{\scriptsize$\pm$0.6} &
        51.8{\scriptsize$\pm$0.9} &
        75.8{\scriptsize$\pm$0.2} & 
        80.5{\scriptsize$\pm$0.5} & 
        88.3{\scriptsize$\pm$1.0} & 
        88.0{\scriptsize$\pm$0.4} &
        - &
        5.8\\
        \midrule \midrule
        GraphMAE & 
        75.5{\scriptsize$\pm$0.7} &
        51.6{\scriptsize$\pm$0.5} &
        75.3{\scriptsize$\pm$0.4} &
        80.3{\scriptsize$\pm$0.5} &
        88.2{\scriptsize$\pm$1.3} &
        88.0{\scriptsize$\pm$0.2} &
        80.4{\scriptsize$\pm$0.3} &
        6.4\\
        \rowcolor{lightgray}
        $+$ CORE$^\text{1}$ & 
        \bf 76.1{\scriptsize$\pm$0.4} &
        \bf 52.8{\scriptsize$\pm$0.2} &
        \bf 76.4{\scriptsize$\pm$0.4} &
        \bf 83.4{\scriptsize$\pm$0.1} &
        \bf 89.4{\scriptsize$\pm$0.5} &
        \bf 89.5{\scriptsize$\pm$0.2} & 
        \bf 83.2{\scriptsize$\pm$0.2} &
        \bf 3.1\\
        \midrule
        GraphMAE2 & 
        76.2{\scriptsize$\pm$0.4} &
        52.7{\scriptsize$\pm$0.5} &
        76.3{\scriptsize$\pm$0.3} &
        84.1{\scriptsize$\pm$0.3} &
        89.6{\scriptsize$\pm$0.3} & 
        88.8{\scriptsize$\pm$0.4} &
        81.8{\scriptsize$\pm$0.5} &
        3.4\\
        \rowcolor{lightgray}
        $+$ CORE$^\text{2}$ & 
        \bf 76.7{\scriptsize$\pm$0.2} &
        \bf 53.0{\scriptsize$\pm$0.3} &
        \bf 77.1{\scriptsize$\pm$0.6} &
        \bf 84.3{\scriptsize$\pm$0.1} &
        \bf 90.1{\scriptsize$\pm$1.1} &
        \bf 92.1{\scriptsize$\pm$0.3} &
        \bf 83.1{\scriptsize$\pm$0.4} &
        \bf 1.4\\
        \bottomrule[1.2pt]
    \end{tabular}}
    \end{threeparttable}
    \label{tab:graph_clf}
\end{table*}

\subsection{Graph Classification}
For graph classification, we draw negative samples from the entire batch of graphs. This approach assumes nodes within the same graph share similarities, which could weaken negative sampling effectiveness. By sampling across the batch, we ensure sufficient distinction among masked nodes, avoiding the contrast of similar nodes within the same graph and enhancing the robustness of the learned representations.

Consistent with the convention in graph classification research, we report results from prior studies where applicable, as shown in Tab.~\ref{tab:graph_clf}. This table presents a comprehensive evaluation of unsupervised representation learning for graph classification, and CORE achieves improvement in all seven datasets compared to the original GraphMAE. Paired t-tests demonstrate a statistically significant improvement of CORE over GraphMAE ($p = 0.0035$).

Specifically, CORE with GraphMAE2 attained the highest accuracy on six datasets, surpassing other methods including supervised ones in IMDB-B, COLLAB, and NCI1 with substantial improvement. These results highlights the effectiveness of our proposed approach showcasing again that integrating contrastive learning within the masked feature reconstruction could extract meaningful information from graph data and advancing graph-level task performance. 

\subsection{Model Analysis}
Here, we assess the CORE framework by examining the effect of its two principal designs, effect of the temperature hyperparameter $T$, and the negative samples size. For brevity, in subsequent results, we default to reporting only the performance of CORE with GraphMAE2. 


The ablation studies in Tab.~\ref{tab:ablation} validate the effectiveness of CORE's two main designs: \textit{masked node-based positive sampling} and \textit{masked node-based negative sampling}, across diverse datasets.

\begin{table}[t]
    \centering
    \small
    \caption{Ablation study on the effect of the design choices in CORE, in node and graph classification performance.}
    \label{tab:ablation}
    \begin{threeparttable}
    \renewcommand\tabcolsep{1.5pt}
    \resizebox{\linewidth}{!}{
    \begin{tabular}{c|ccccccccc}
    \toprule[1.2pt]
    \multicolumn{10}{c}{\textbf{Node Classification Accuracy}}\\
    \midrule
    & Cora & Citeseer & PubMed & Wiki-CS & Computer & Photo & CS & Physics & PPI\\
    \midrule
    {\scriptsize w/o} MNPS & 
    83.8 & 70.9 & 57.4 & 79.1 & 81.8 & 91.9 & \bf 93.3 & 95.3 & 73.2\\
    {\scriptsize w/o} MNNS & 
    84.1 & 74.0 & 79.1 & 80.0 & 84.3 & 93.3 & 92.9 & 95.5 & 70.4\\
    {\scriptsize w/o} both & 
    83.5 & 71.0 & 79.9 & 80.2 & 90.9 & 93.8 & 93.3 & \bf 95.7 & \bf 75.5\\
    \rowcolor{lightgray}
    CORE & 
    \bf 84.2 & \bf 74.1 & \bf 81.3 & \bf 80.2 & \bf 91.0 & \bf 93.9 & 93.2 & 95.6 & 75.2\\
    \midrule \midrule
    \end{tabular}}

    \renewcommand\tabcolsep{1.5pt}
    \renewcommand\arraystretch{1.05}
    \resizebox{\linewidth}{!}{
    \begin{tabular}{c|ccccccc}
    \multicolumn{8}{c}{\textbf{Graph Classification Accuracy}}\\
    \midrule
    & IMDB-B & IMDB-M & PROTEINS & COLLAB & MUTAG & REDDIT-B & NCI1\\
    \midrule
    {\scriptsize w/o} MNPS & 
    76.4 & 52.4 & 77.1 & 81.1 & 87.1 & 91.9 & 82.6 \\
    {\scriptsize w/o} MNNS & 
    76.1 & 52.8 & 77.1 & 78.7 & 82.0 & 88.6 & 80.2 \\
    {\scriptsize w/o} both & 
    76.1 & 52.5 & 77.1 & 82.4 & 85.7 & 91.4 & 82.7 \\
    \rowcolor{lightgray}
    CORE & 
    \bf 76.7 & \bf 53.0 & \bf 77.1 & \bf 84.3 & \bf 90.1 & \bf 92.1 & \bf 83.1 \\
    \bottomrule[1.2pt]
    \end{tabular}}
    \end{threeparttable}
\end{table}

In node classification, omitting \textit{masked node-based positive sampling} (‘w/o MNPS’), where all nodes are used as anchors, results in decreased performance on datasets like Cora, Citeseer, and PubMed, compared to CORE, which exclusively uses masked nodes as anchors. Similarly, without \textit{masked node-based negative sampling} (‘w/o MNNS’), where negative samples are drawn from all nodes instead of just masked nodes, accuracy drops, particularly on datasets like Wiki-CS, Computer, and Photo. The variant ‘w/o both’, which uses all nodes as anchors and negative samples from the entire graph (equivalent to standard node-level GCL), performs better than the other two variants but still falls short of CORE across most datasets. However, CORE slightly underperforms on CS and Physics, possibly due to the large raw feature dimensions of 6,805 and 8,415 in these datasets, which may cause the model to overfit to the masked nodes. Preprocessing techniques like principal component analysis could help address this issue.

For graph classification, omitting either or both design choices similarly results in performance declines across all datasets. These results underscore the value of focusing on masked nodes for generating positive pairs and selecting negative samples, collectively enhancing learning and improving graph representations.

Overall, our ablation studies underscore the strength of using masked nodes as both positive and negative samples in CORE. These design choices are thoroughly validated by our experiments, demonstrating consistent performance improvements across node and graph classification tasks.

\begin{figure}[t]
\centering
\includegraphics[width=0.95\linewidth]{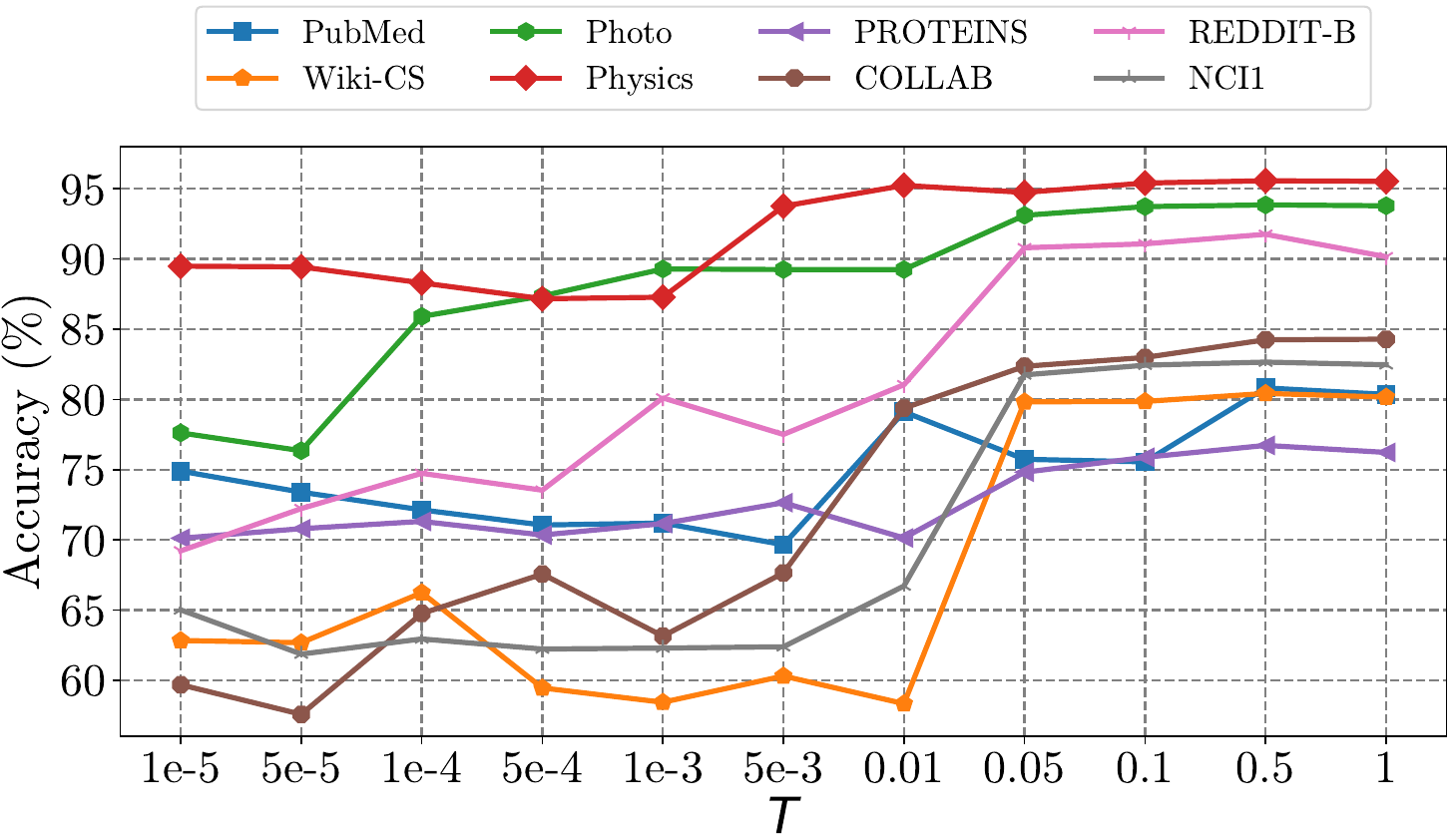}
\caption{Effect of temperature $T$.}
\label{fig:temperatures}
\end{figure}

\stitle{Effect of temperature $T$.} We investigated the effect of the temperature parameter $T$ on model performance, as shown in Fig.~\ref{fig:temperatures}. Lower temperatures reduce the utilization of negative samples, leading the model to focus more on positive pairs and resulting in poorer performance. In contrast, higher temperature generally improves performance by better leveraging negative samples to enhance the model's discriminative capabilities. These findings underscore the importance of negative samples, as guided by theoretical insights, and highlight the value of our framework in seamlessly integrating contrastive learning into MFR. The datasets in Fig.~\ref{fig:temperatures} include four for node classification and four for graph classification, clearly illustrating the observed trends. 



\stitle{Effect of negative sample size $|\mathcal{M}|$.}
\label{sec:num_neg}
We examined the effect of varying the number of negative samples on model performance, with the results presented in Tab.~\ref{tab:negative_samples}. Our results indicate that increasing the number of negative samples generally leads to improved performance, with optimal results varying by dataset. For node classification, the highest accuracy could often be achieved with larger pools of negative samples, as seen in datasets like Citeseer, Cora, and PPI. 

Similarly for graph classification tasks shown in Tab.~\ref{tab:negative_samples_graph}, performance generally improves with more negative samples, particularly in datasets like IMDB-M, PROTEINS, and REDDIT-B, where the highest accuracies are achieved with larger sample sizes (64-128 samples). However, some datasets like COLLAB and MUTAG achieve optimal performance with smaller negative sample sizes (8 and 16 samples respectively), suggesting that the optimal number varies across different graph structures and tasks.

This suggests that CORE scales effectively with the number of negative samples, likely due to our strategy of exclusively selecting masked nodes as negative samples. This approach allows CORE to leverage diverse negative contexts for learning without compromising the quality of the learned representations.

\begin{table}[t]
    \small
    \centering
    \caption{Performance under varying negative sample sizes on node classification.}
    \begin{threeparttable}
    \renewcommand\tabcolsep{5pt}
    \resizebox{\linewidth}{!}{
    \begin{tabular}{@{}c|cccccccccc@{}}
    \toprule[1.2pt]
    $|\mathcal{M}|$ & 0 & 2 & 4 & 8 & 16 & 32 & 64 & 128 & 256 \\
    \midrule
    Cora & 83.2 & 83.5 & 83.5 & 83.7 & 84.0 & 84.2 & 84.4 & 84.3 & \bf 84.5\\
    Citeseer & 73.2 & 73.1 & 73.2 & 73.5 & 73.4 & 73.7 & \bf 73.9 & \bf 73.9 & \bf 73.9\\
    PubMed & 78.8 & 80.1 & 81.1 & 81.0 & 80.8 & 81.2 & \bf 81.5 & 81.0 & 81.0 \\
    Wiki-CS & 80.2 & 80.2 & \bf 80.3 & \bf 80.3 & \bf 80.3 & \bf 80.3 & \bf 80.3 & \bf 80.3 & \bf 80.3 \\
    Computer & \bf 91.2 & 90.8 & 91.1 & 90.1 & \bf 91.2 & \bf 91.2 & \bf 91.2 & 91.0 & \bf 91.2\\
    Photo & 93.7 & 93.7 & 93.7 & 93.6 & 93.7 & \bf 93.9 & 93.8 & 93.8 & 93.8\\
    CS & 93.0 & 93.3 & 93.3 & 93.2 & \bf 93.3 & 93.2 & 93.2 & 93.2 & 93.3 \\
    Physics & 95.3 & \bf 95.6 & \bf 95.6 & \bf 95.6 & 95.5 & 95.5 & 95.5 & 95.5 & \bf 95.6\\
    PPI & 74.4 & 72.1 & 73.7 & 74.2 & 74.7 & 75.0 & 75.1 & 75.2 & \bf 75.4\\
    \bottomrule[1.2pt]
    \end{tabular}
    }
    \end{threeparttable}
    \label{tab:negative_samples}
\end{table}

\begin{table}[t]
    \small
    \centering
    \caption{Performance under varying negative sample sizes on graph classification.}
    \begin{threeparttable}
    \renewcommand\tabcolsep{3pt}
    \renewcommand\arraystretch{1.05}
    \resizebox{1\linewidth}{!}{
    \begin{tabular}{c|ccccccccc}
    \toprule[1.2pt]
    $|\mathcal{M}|$ & 0 & 2 & 4 & 8 & 16 & 32 & 64 & 128 & 256\\
    \midrule
    IMDB-B   & 74.6  & 76.5  & 76.3  & \bf 76.6  & 76.5  & \bf 76.6  & \bf 76.6  & 76.5  & 76.5 \\
    IMDB-M   & 52.2  & 52.7  & 52.7  & 52.8  & 52.7  & 52.6  & \bf 52.9  & 52.5  & 52.6 \\
    PROTEINS & 75.3  & 76.5  & 76.6  & 76.8  & 76.9  & 76.9  & 77.0  & \bf 77.1  & 77.0 \\
    COLLAB   & 83.7  & \bf 84.3  & \bf 84.3  & \bf 84.3  & 84.2  & 83.8  & 84.2  & 84.0  & 83.9 \\
    MUTAG    & 80.2  & 89.7  & 89.6  & 89.5  & \bf 90.0  & 89.2  & 89.4  & 88.8  & 88.1 \\
    REDDIT-B & 78.9  & 91.0  & 91.3  & 91.4  & 91.9  & 91.7  & 91.8  & \bf 92.1  & 91.7 \\
    NCI1     & 76.4  & 81.7  & 82.2  & 82.4  & 82.5  & 83.0  & \bf 83.1  & 82.8  & 82.6 \\

    \bottomrule[1.2pt]
    \end{tabular}
    }
    \end{threeparttable}
    \label{tab:negative_samples_graph}
\end{table}



It is worth noting that setting $T$ near-zero differs from using zero negative samples ($|\mathcal{M}|=0$). With $|\mathcal{M}|=0$, the loss becomes pure reconstruction focusing on positive pairs. In contrast, when $T$ is very small with negative samples present, the negative pairs term still exists but becomes numerically unstable, leading to degraded performance. This explains the different behaviors observed in Fig.~\ref{fig:temperatures} and Tab.~\ref{tab:negative_samples}, and why we use standard temperature values for stable training.

\subsection{Scalability to Large-Scale Graphs}
To demonstrate CORE's scalability, we conducted experiments on two large-scale datasets: ogbn-arxiv (169K nodes) and ogbn-products (2.4M nodes). Following GraphMAE's efficient implementation, we trained and tested on sampled subgraphs, ensuring CORE maintains similar memory efficiency as GraphMAE.

\begin{table}[t]
    \centering
    \small
    \caption{Performance on large-scale graph datasets (node classification accuracy \%).}
    \begin{tabular}{l|cc}
    \toprule[1.2pt]
    Method & ogbn-arxiv & ogbn-products \\
    \midrule
    GraphMAE & 71.03 & 78.89 \\
    + CORE & \textbf{71.17} & \textbf{79.82} \\
    \midrule
    GraphMAE2 & 71.32 & \textbf{80.09} \\
    + CORE & \textbf{71.57} & 79.52 \\
    \bottomrule[1.2pt]
    \end{tabular}
    \label{tab:large_scale}
\end{table}

Table~\ref{tab:large_scale} shows that CORE consistently improves over GraphMAE (+0.14\% to +0.93\%) and maintains competitive performance with GraphMAE2, demonstrating its effectiveness on large graphs.

The computational complexity of CORE is $O(|\widetilde{V}||\mathcal{M}_i|d)$, where $|\widetilde{V}|$ is the number of masked nodes, $|\mathcal{M}_i|$ is the number of negative samples per masked node, and $d$ is the feature dimension. Since both $|\widetilde{V}|$ and $|\mathcal{M}_i|$ are typically much smaller than the total number of nodes $N$, CORE maintains good efficiency even on large graphs.

\begin{table}[t]
    \centering
    \small
    \caption{Training time per epoch (seconds).}
    \begin{tabular}{l|cc}
    \toprule[1.2pt]
    Method & Physics (34K nodes) & ogbn-arxiv (169K nodes) \\
    \midrule
    GraphMAE2 & 0.25 & 0.93 \\
    + CORE & 0.35 & 1.12 \\
    Overhead & +40\% & +20\% \\
    \bottomrule[1.2pt]
    \end{tabular}
    \label{tab:time_comparison}
\end{table}

As shown in Table~\ref{tab:time_comparison}, the computational overhead of CORE decreases with larger graphs (from 40\% on Physics to 20\% on ogbn-arxiv) due to more efficient batch processing with subgraphs. This demonstrates that CORE offers a favorable trade-off between performance gains and computational cost, particularly for large-scale graphs.

\section{Preliminary Analysis}
\label{sec:prelim_analysis}
Before developing our CORE framework, we conducted a preliminary analysis to explore the relationship between masked feature reconstruction and node-level graph contrastive learning. This analysis served to empirically validate our theoretical insights but is distinct from our proposed method.

\begin{figure}[t]
    \centering
    \includegraphics[width=0.6\linewidth]{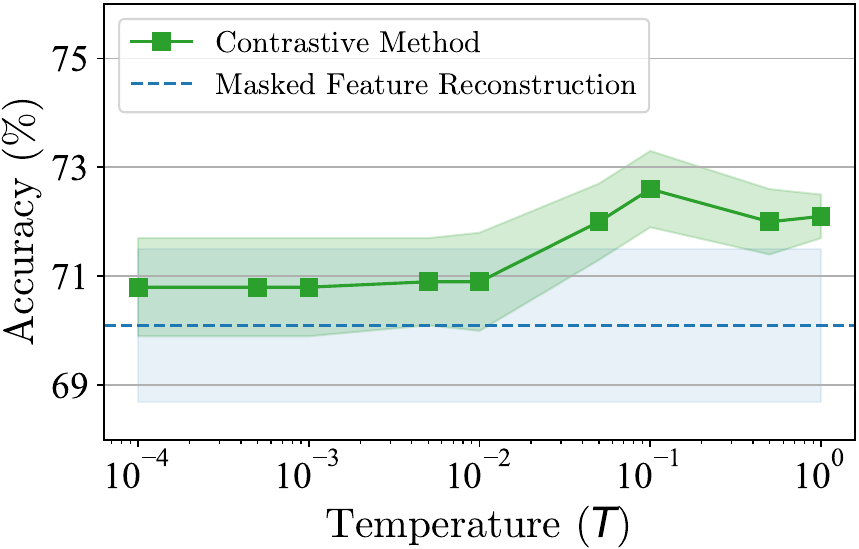}
    \caption{Relationship between temperature in node-level GCL and its performance compared to MFR on Citeseer dataset. Shaded area indicates standard deviation.}
    \label{fig:prelim}
\end{figure}



Figure~\ref{fig:prelim} shows node-level GCL performance approaching that of MFR as temperature decreases on Citeseer. This is \textbf{not our proposed CORE method versus MFR}, but rather empirical validation of the theoretical connection in Section~\ref{sec:theoretical_insights}. Using a simple setup with ``Masking'' and ``No-Operation'' views, the trend confirms that at lower temperatures, positive pairs dominate the contrastive loss, creating an MFR-like scenario without explicit negative samples.

\section{Conclusion}
In this paper, we presented Contrastive Masked Feature Reconstruction (CORE), a novel self-supervised graph learning framework that enhances MFR by integrating it with contrastive learning principles. This integration is motivated by our theoretical insights demonstrating the convergence of MFR and node-level GCL objectives under certain conditions. Building on their compatibility, we developed a framework that leverages masked nodes as both positive and negative samples, preserving the benefits of feature reconstruction while encouraging better contextual understanding and node differentiation. Our extensive experiments validate this approach, demonstrating that CORE consistently outperforms state-of-the-art methods like GraphMAE in both node and graph classification tasks. These findings pave the way for advancements in graph SSL by showing how different strategies can be effectively combined.

\bibliographystyle{ACM-Reference-Format}
\bibliography{core}

\end{document}